\def\year{2021}\relax
\documentclass[letterpaper]{article} 
\usepackage{aaai21}  
\usepackage{times}  
\usepackage{helvet} 
\usepackage{courier}  
\usepackage[hyphens]{url}  
\usepackage{graphicx} 
\usepackage{natbib}  
\usepackage{caption} 
\usepackage[switch]{lineno}  %

\usepackage{multirow}
\usepackage{bm}
\usepackage{subfigure}
\usepackage{algorithm}
\usepackage{algorithmic}
\usepackage{booktabs}
\usepackage{amsmath}
\usepackage{amsthm}
\usepackage{amsfonts}
\usepackage{listings}
\newtheorem{cor}{Corollary}[]

\title{Learning Robust Representation for Clustering through Locality Preserving Variational Discriminative Network}

\author{
Ruixuan Luo\textsuperscript{\rm \dag}, 
Wei Li\textsuperscript{\rm \ddag}, 
Zhiyuan Zhang\textsuperscript{\rm \ddag}, 
Ruihan Bao\textsuperscript{\rm $\diamondsuit$}, 
Keiko Harimoto\textsuperscript{\rm $\diamondsuit$}, 
Xu Sun\textsuperscript{\rm \dag,\ddag}\\
}
\affiliations{
\textsuperscript{\rm \dag}Center for Data Science, Peking University\\
\textsuperscript{\rm \ddag}MOE Key Laboratory of Computational Linguistic, School of EECS, Peking University\\
\textsuperscript{\rm $\diamondsuit$}Mizuho Securities Co.,Ltd\\
\{luoruixuan97, liweitj47, zzy1210\}@pku.edu.cn,
\{ruihan.bao, keiko.harimoto\}@mizuho-sc.com, xusun@pku.edu.cn
}

\begin{document}

\maketitle

\begin{abstract}

Clustering is one of the fundamental problems in unsupervised learning. 
Recent deep learning based methods focus on learning clustering oriented representations. Among those methods, Variational Deep Embedding achieves great success in various clustering tasks by specifying a Gaussian Mixture prior to the latent space. However, VaDE suffers from two problems: 1) it is fragile to the input noise; 2) it ignores the locality information between the neighboring data points. 
In this paper, we propose a joint learning framework that improves VaDE with a robust embedding discriminator and a local structure constraint, which are both helpful to improve the robustness of our model. Experiment results on various vision and textual datasets demonstrate that our method outperforms the state-of-the-art baseline models in all metrics. Further detailed analysis shows that our proposed model is very robust to the adversarial inputs, which is a desirable property for practical applications.


\end{abstract}

\section{Introduction}

Clustering is an essential problem in unsupervised learning, which aims to group unlabeled data based on their similarities. Over the past decades, several traditional clustering methods such as k-means and Gaussian Mixture Model (GMM) \citep{kohonen1990self,hartigan1979algorithm,ester1996density} have been established, which measure the distances of data points using Euclidean distance and minimize the within-class variances. However, as the dimension of the input data grows in the feature space, the clustering with Euclidean distance will deteriorate significantly. In order to solve the problem, various dimension reduction methods \citep{turk1991eigenfaces, belhumeur1997eigenfaces, donoho2003hessian} have been proposed to convert data points from the feature space to a lower-dimensional space as a pre-processing step for clustering.

Due to the recent success, deep learning based clustering receives lots of attention. Earlier approaches adopt deterministic models such as Autoencoder to learn a latent representation by minimizing the reconstruction error. Furthermore, constraints are imposed to make the latent representation more suitable for the clustering \citep{xie2016unsupervised,yang2017towards}. 
Other attempts build clustering methods based on generative models because of their ability to learn the inherent data structure.
\citet{mukherjee2019clustergan} design a GAN based model which contains an inverse-mapping network and a clustering-specific loss. Their model achieves high scores on various tasks, but is vulnerable to the unbalanced data. \citet{jiang2016variational} propose Variational Deep Embedding (VaDE) that extends VAE by adding a GMM constraint to the latent layer. 
While VaDE achieves the state-of-the-art performance in various clustering tasks, it suffers from two major drawbacks: 1), as VaDE trains a network that tries to recover the original inputs, the latent representations are vulnerable to the input noise; 2), it only models the global structure of the latent representations, 
while ignoring the local structure between two latent variables which carries important information.

In this paper, we resolve the problems of VaDE by proposing a \textbf{L}ocality \textbf{P}reserving \textbf{V}ariational \textbf{D}iscriminative \textbf{N}etwork (LPVDN) that encodes robust global and local data structure in the latent representations used for clustering, which is stable for noise and unbalanced input. Specifically, our work first models the global data structure using a GMM regularized Variational Autoencoder (VAE), because VAE is good at capturing the global structure of data distribution. Then we improve the \textbf{robustness} of the \textbf{global} structure modeling by introducing an embedding discriminator that maximizes the mutual information between the input data and the latent representations, because maximizing the mutual information can improve the robustness to the input noise.\footnote{A mathematical proof is given in the Appendix.} 
After the global representations are available,
we encode the \textbf{local} structure information by modeling pairwise distance probability and minimizing the KL-divergence of the distribution. We assume that the global structure is more sensitive to the adversarial noise, while local structures can be helpful to reduce the impact of the noise.

Specifically, to capture the local structures, we propose a locality preserving network module which is inspired by t-SNE \citep{maaten2008visualizing} due to its effectiveness and simplicity. It employs a Student t-distribution with one degree of freedom in the low-dimensional map and uses it to fit the joint probability in the high-dimensional space. The optimization objective attracts similar points while introduces strong but finite repulsion between dissimilar points. 

We perform extensive experiments on both vision and textual clustering datasets. Experiments show that our model achieves the state-of-the-art results on all the metrics. Further analysis testifies the robustness under different degree of data poisoning of our model and reveals its potential for data dimension reduction and visualization.

\begin{figure}[t]
    \centering
    \includegraphics[width=\linewidth]{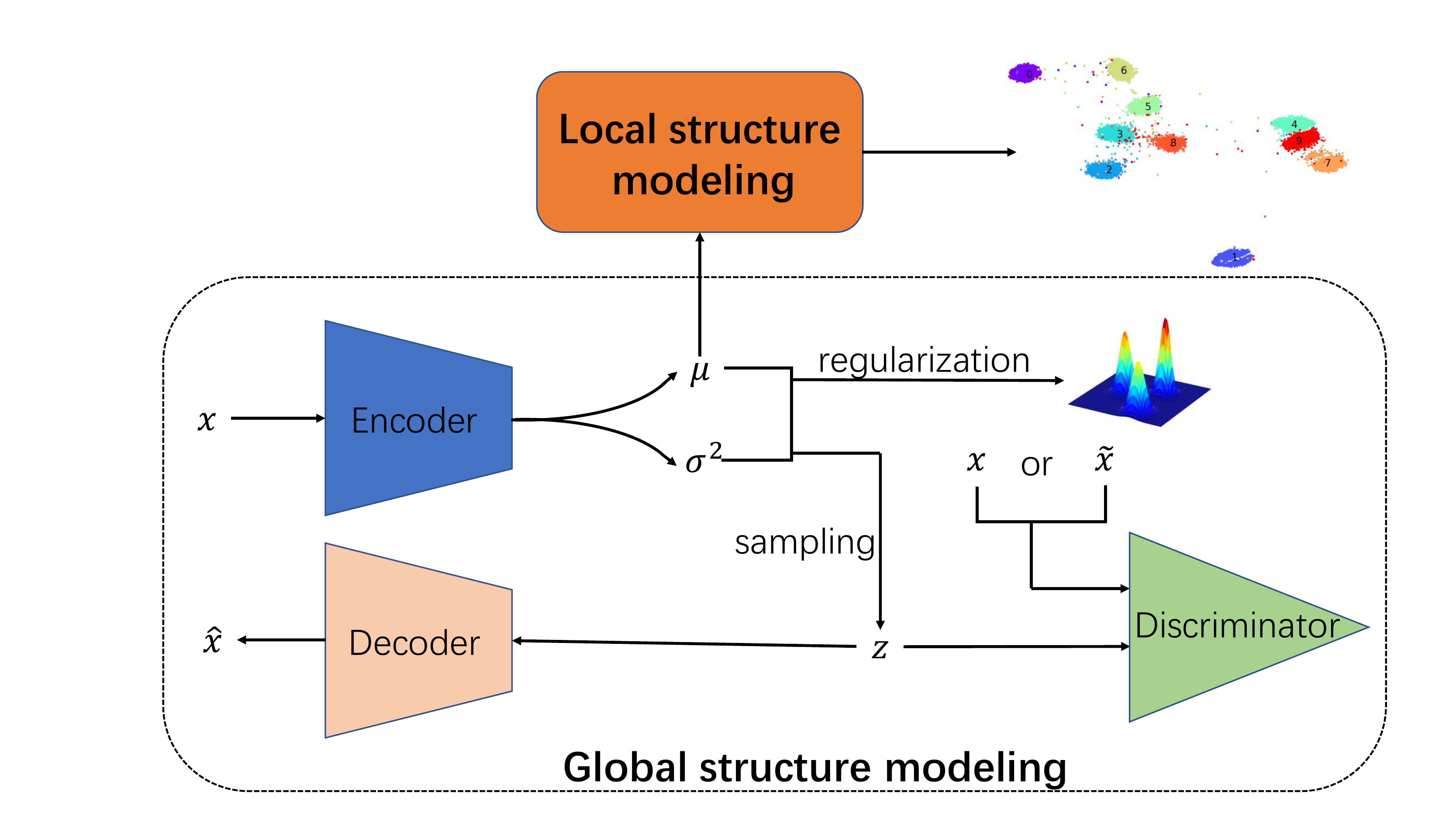}
    \caption{Illustration of Locality Preserving Variational Discriminative Network (LPVDN). Our model retains both global (lower part) and local (upper part) structure. Latent variables $\mu$ and $\sigma$ are regularized to a mixture of Gaussian distribution.  
    }
    \label{fig:model}
\end{figure}

We conclude our main contributions as follows:
\begin{itemize}
    \item 
    We propose a robust locality preserving variational discriminative network that retains both global and local data structure, which is also robust to the noise in the input. \footnote{Our code is released at https://github.com/lancopku/LPVDN}
    \item 
    Extensive experiments on both vision and textual benchmark datasets show that our proposed model outperforms the state-of-the-art clustering methods. Analysis testifies the superiority of the robustness of our model over previous state-of-the-art methods.
    
    \item 
    Further analysis shows that our model's performance is steady even when the dimension of the latent representations reduces to two, which means that our model can be served as a promising candidate for data visualization.
\end{itemize}

\section{Related Work}

Recently, deep learning-based clustering methods attract lots of attention due to its success in various problems.

One branch of the popular methods are based on Autoencoders, in which a latent feature representation can be learned by minimizing reconstruction errors. Specifically, \citet{yang2017towards} propose DCN by incorporating the K-means loss with the reconstruction loss to perform clustering and feature embedding simultaneously. \citet{huang2014deep} introduce DEC with a KL-divergence loss between soft assignments of the clusters and an auxiliary distribution to minimize within-clustering distance. In order to achieve better performance, the paper initializes the parameters with general Autoencoder training, which is adopted as the major practice for other clustering networks using deep learning. \citet{ji2017deep} apply the idea of subspace clustering and introduce a self-expressive layer to the Autoencoder, the trained parameters of the self-expressive layer are used to construct an affinity matrix used by spectral clustering. Leveraging recent progress in image recognition, \citet{ghasedi2017deep} adopt a convolutional Autoencoder with a softmax layer stacked on top of the encoder. \citet{yang2019deep} extend the idea of convolutional Autoencoder by proposing a joint learning framework of spectral clustering and discriminative embedding which is regularized by a normal distribution. As convolutional Autoencoder can extract important image features. Both methods work extremely well for image clustering tasks, but the performance often drops significantly for non-image clustering tasks.


Driven by the success of VAE \citep{kingma2013auto} and GAN \citep{goodfellow2014generative}, generative models become favourable choices for clustering due to its ability to capture the data distribution. Among those methods, Variational Deep Embedding \citep{jiang2016variational} extends VAE by assuming latent variables following a mixture of Gaussian distribution, where the means and variances of the Gaussian components are trainable. Deep Adversarial Clustering \citep{harchaoui2017deep} proposes an adversarial Autoencoder, which utilizes an adversarial training procedure to match the aggregated posterior of the latent representation with the prior distribution. CategorialGAN \citep{springenberg2015unsupervised} generalizes the GAN framework to multiple classes. InfoGAN \citep{chen2016infogan} maximizes the mutual information between a subset of the latent variables and the observation, which achieves disentangled representations for clustering. ClusterGAN \citep{mukherjee2019clustergan} samples latent variables from a mixture of one-hot and continuous latent variables, and the inverse network is trained jointly with a clustering specific loss. However, as it requires sampling one-hot vectors from a uniform distribution of the clusters, when the data is highly unbalanced, the performance will experience a serious decline.  

\section{Locality Preserving Variational Discriminative Network (LPVDN)}

In this section, we describe our proposed \textbf{L}ocality \textbf{P}reserving \textbf{V}ariational \textbf{D}iscriminative \textbf{N}etwork (LPVDN). 

The idea of our method is to joint learn robust global and local structures in the latent space. The global structure modeling is based on the VaDE, which achieves good results in clustering with a Variational Autoencoder and Gaussian Mixture Model (GMM), however, it is vulnerable to the noise in the original input data space and fails to model the local structure between data points. In order to tackle these problems, we propose a novel model that not only improves the robustness of the global structure, but also encodes the local structure between data points. A sketch of our model is given in Fig. \ref{fig:model}.

More specifically, our methods can be explained using the following optimization objective:
\begin{small}
\begin{equation}
  \label{eq:overview}
  \begin{aligned}
   \min_{\phi, \theta, \psi, \rho} \sum_{i=1}^{N}{} & \big(L_\text{G}(x_i;\phi,\theta) \\
   &  +\alpha_{0}L_\text{MI}(x_i;\phi, \rho)+\alpha_{1}L_\text{LP}(x_i, x_j;\phi,\psi)\big) \\
  \end{aligned}
\end{equation} 
\end{small}
where $x_i$ and $x_j$ are input data points, $n$ is the total number of the data, $\phi$, $\theta$, $\psi$, $\rho$ are trainable parameters. 

As expressed in Eqn.\ref{eq:overview}, our model can be decomposed into three parts: $L_\text{G}(x_i;\phi,\theta)$ captures the global structure using a Variational Autoencoder with GMM Prior.
$L_\text{MI}(x_i;\phi, \rho)$ applies an embedding discriminator to improve the robustness of the global structure modeling.
Finally, $L_\text{LP}(x_i, x_j;\phi, \psi)$ models the local structure by minimizing the KL-divergence of the pairwise distance distribution. $\alpha_{0}$ and $\alpha_{1}$ are hyper-parameters which try to balance different parts. After the training phase, as the final latent representations comprise robust global and local information, we use it as the input for general clustering methods such as K-means or GMM.

\subsection{Global Structure Modeling $L_\text{G}(x_i;\phi,\theta)$}
To model the global structure, we take advantage of Variational Deep Embedding (VaDE) \citep{jiang2016variational}, which extends VAE with GMM to model the distribution of latent representations. 
Specifically, both VAE and VaDE maximize the log probability of the given data point:
\begin{equation}
\max_{\phi,\theta}\sum_{i=1}^{N}\log{p(x_i;\phi, \theta)}    
\end{equation}
As the optimization is non-trivial, an evidence lower bound (ELBO) \citep{kingma2013auto,jiang2016variational} is proposed and maximized instead. In order to understand the ELBO, we further decompose it as:
\begin{equation}
\begin{aligned}
\label{eq:vde_decompose}
   L_\text{ELBO}(x_i)= & E_{q(z,c|x_i)}[\log{p(x_i|z)}] \\
                 &-D_\text{KL}(q(z,c|x_i)||p(z,c))
\end{aligned}
\end{equation}
where $L_\text{ELBO}$ is the ELBO loss to be optimized. Term $x_i$ denotes the input data, $z$ represents the latent variables, $c$ is the GMM cluster in the latent layer. 
Eqn.\ref{eq:vde_decompose} shows that the ELBO loss consists of two terms. The first term corresponds to the reconstruction loss, which encourages the network to better approximate the original data. The second term is the regularization term, which requires the distribution of latent variables $z$ to be as close as Mixture-of-Gaussian distribution $p(\bm{z},c)$. Since ELBO preserves the global distribution of latent variables, our global structure term $L_\text{G}(x_i;\phi,\theta)$ is set to the negative ELBO loss (since we minimize the objective), precisely, $L_\text{G}(x_i;\phi,\theta)=-L_\text{ELBO}(x_i)$.

In practice, following the derivation of \citet{jiang2016variational}, $L_\text{G}(x_i;\phi,\theta)$ is computed as
\begin{small}
\begin{equation}
\begin{aligned}
&{}L_\text{G}(x_i;\phi,\theta) \\
= &  -\sum_{d=1}^{D}{x_i^d\log{\bm{\mu}_{x_i}|_d}}+(1-x_i^d)log(1-\bm{\mu}_{x_i}|_d)\\
& + \sum_{c=1}^{K}{\gamma_{ic}\sum_{j=1}^{J}{\big(\log{\bm{\sigma}_{c}^2|_j}}+\dfrac{\bm{\widetilde{\sigma}}_i^2|_j}{\bm{\sigma}_{c}^2|_j}+\dfrac{{(\bm{\widetilde{\mu}}_i|_j-\mu_c|_j)^2}}{\bm{\sigma}_{c}^2|_j}\big)}\\
&-\sum_{c=1}^{K}{\gamma_{ic}\log{\dfrac{\pi_{ic}}{\gamma_{ic}}}}-\dfrac{1}{2}\sum_{j=1}^{J}(1+\log{\bm{\widetilde{\sigma}}_i^2|_j})
\end{aligned}
\end{equation}
\end{small}
where $D$ is the dimension of input $x_i$ and reconstructed outputs $\bm{\mu}_{x_i}$. $x_i^d$ is the $d$-th element of $x_i$, $*|_j$ denotes the $j$-th element of $*$. $\gamma_{ic}$ represents $q(c|z_i)$, which is the GMM distribution for latent variable $z_i$. Moreover,
\begin{align}
     [\bm{\widetilde{\mu}}_i, \log\bm{\widetilde{\sigma}}_i^2] &= g_{g}(x_i; \phi) \label{eq:latent_mu}\\
     z_i &= \widetilde{\bm{\mu}}_i + \widetilde{\bm{\sigma}}_i\circ{\epsilon}_i \label{eq:z}\\
      \bm{\mu}_{x_i} &= f_{g}(z_i;\theta) 
\end{align}
where $g_{g}(*;\phi)$ and $f_{g}(*;\theta)$ are two neural networks parameterized by $\phi$ and $\theta$, $z_i$ is generated from the latent distribution $\mathcal{N}(\bm{\widetilde{\mu}}_i, \bm{\widetilde{\sigma}}_i ^2)$, $\epsilon_i$ is a vector with elements drawn from independent normal distribution and $\circ$ denotes the element-wise multiplication. $\bm{\mu}_{x_i}$ is the reconstructed data of $x_i$.

\subsection{Robust Embedding Discriminator $L_\text{MI}(x_i;\phi,\rho)$}

In Eqn.\ref{eq:vde_decompose}, it is essential to reconstruct the original data, which is achieved by minimizing the cross entropy loss or mean square error between raw data and the latent representations. However, such objectives are too strict, which would make the model vulnerable to the input noise. In order to improve the robustness of our model, we reformulate the reconstruction loss over all input data samples (for the simplicity, we omit the variable $c$),
\begin{small}
\begin{equation}\label{eq:elbo2mi}
  \begin{aligned}
     &E_{x\sim P(X)}\big[E_{q(z|x)}[\log(p(x|z))]\big]  \\
     = &\int_{}\int_{}p(x)q(z|x)\log(p(x|z))dzdx \\
     = &\int_{}\int_{}p(x)q(z|x)\log\big(\frac{q(z|x)p(x)}{q(z)}\big)dzdx \\
     = &\int_{}\int_{}p(x)q(z|x)\log\big(\frac{q(z|x)p(x)}{q(z)p(x)}\big)dzdx \\
        & + \int_{}\int_{}p(x)q(z|x)\log(p(x))dzdx \\
     = & D_\text{KL}(Q(Z|X)P(X)||Q(Z)P(X))-H(X) \\
     = & I(X, Z)-H(X)
  \end{aligned}
\end{equation}
\end{small}
where $X=\{x_1, ..., x_n\}$ represent the set of all the input samples, $Z$ represents the corresponding latent variables. Since $H(X)$ solely depends on the input data, the maximization of the first term in Eqn.\ref{eq:vde_decompose} is equivalent to maximize the mutual information $I(X,Z)$. 

To estimate the mutual information $I(X,Z)$, \citet{hjelm2018learning} introduce a JS-divergence based approximation,
\begin{equation}
\label{eq:JSD}
  I^\text{(JSD)}(X,Z)=D_\text{JS}(Q(Z|X)P(X),Q(Z)P(X))
\end{equation}

One way to compute the above JS-divergence in Eqn. \ref{eq:JSD} is using a discriminator $D$ \citep{nowozin2016f, hjelm2018learning,yang2019deep}:
\begin{equation}
\begin{aligned}
\label{discriminator}
  D_\text{JS}(P(X),Q(X)) & = \frac{1}{2}\max_D \{E_{x\sim P(X)}[\log (\sigma(D(x)))]\\
    &+E_{x\sim Q(X)}[\log (1-\sigma(D(x)))]\}\\
    &+\log 2
\end{aligned}
\end{equation}

Following Eqn. \ref{eq:JSD} and Eqn.\ref{discriminator},  we utilize a neural network to serve as the robust embedding discriminator,
\begin{equation}
\begin{aligned}\label{eq:LMI}
  &L_\text{MI}(x_i;\phi,\rho)=-\log[\sigma(D(x_i, z_i; \rho)]\\
  &-E_{(x_j,z_i)\sim p(z_i)p(x_j)}[\log (1-\sigma(D(x_j,z_i;\rho)))]
  \end{aligned}
\end{equation}
where $D(*;\rho)$ is a neural network parameterized by $\rho$. Term $z_i$ is the latent variable generated from $x_i$ by Eqn.\ref{eq:latent_mu} and Eqn.\ref{eq:z}. Discriminator $D(*;\rho)$ outputs a binary value indicating whether $z$ is generated from $x$. The second term in Eqn.\ref{eq:LMI} is negative sampling estimation to train the discriminator. In practice, we randomly draw samples from the datasets to form the negative pairs with $z_i$.\footnote{In the appendix, we further provide a theoretical proof for the effectiveness of the proposed robust discriminator.}

One thing should be noted is that $L_\text{MI}(x_i;\phi,\rho)$ cannot replace the original reconstruction loss in Eqn.\ref{eq:vde_decompose}. This is because Eqn.\ref{eq:JSD} is valid only when $P(X)$ and $Q(X)$ are close. As we maximize the mutual information $I(X,Z)$ during training, $P(X)$ and $Q(X)$ will be gradually varied and Eqn.\ref{eq:JSD} will no longer hold. In order to avoid this problem, we combine the reconstruction term and the robust embedding discriminator in our method to achieve better clustering performance (Eqn.\ref{eq:overview}).


\begin{table*}[ht]
\centering
\begin{tabular}{lllllll}
\toprule
\multirow{2}{*}{method} & \multicolumn{3}{c}{MNIST} & \multicolumn{3}{c}{Fashion-MNIST}  \\
& ACC & NMI & ARI & ACC & NMI & ARI  \\
\midrule
K-means      & 0.5850 & 0.4998 & 0.3652 & 0.5542 & 0.5117 & 0.3477      \\
AE+K-means   & 0.7147 & 0.6610 & 0.5593 & 0.6240 & 0.6360 & 0.4440    \\
VAE+K-means  & 0.8583 & 0.7432 & 0.7249 & 0.6049 & 0.5629 & 0.4082     \\
DEC \citep{xie2016unsupervised} & 0.8971 & 0.8631 & 0.8322 & 0.6303 & 0.6560 & 0.4987   \\
ClusterGAN \citep{mukherjee2019clustergan}  & 0.9596 & 0.9051 & 0.9134 & 0.6374 & 0.6423 & 0.5067      \\
VaDE \citep{jiang2016variational} & 0.9446* & 0.8761 & 0.8822 & 0.6383 & 0.6434 & 0.4955    \\
LPVDN (Proposed)    & \textbf{0.9716} & \textbf{0.9276} & \textbf{0.9385} & \textbf{0.6837} & \textbf{0.6912} & \textbf{0.5151}      \\
\bottomrule
\end{tabular}
\caption{Evaluation on vision datasets of MNIST and Fashion-MNIST for different algorithms using ACC, NMI and ARI. * is taken from \citep{jiang2016variational}, other results are generated by the released code.}
\label{tab:visual_result}
\end{table*}

\subsection{Local Structure Modeling $L_\text{LP}(x_i, x_j;\phi, \psi)$}

While the global structure models the latent space based on clustering distribution (e.g. GMM), such methods may generate unsatisfactory results for the data around the cluster boundary. To solve the problem, we introduce a local structure model that captures the pairwise relationships by minimizing the KL-divergence on the distance distribution. 

Inspired by \citet{maaten2008visualizing}, we define the following pairwise distance distribution $p_{ij}$ in the latent space:
\begin{equation}
    p_{j|i} = \frac{\exp(-\|\bm{\widetilde{\mu}}_i-\bm{\widetilde{\mu}}_j\|^2/2\eta_i^2)}{\sum_{k\not=i}\exp(-\|\bm{\widetilde{\mu}}_k-\bm{\widetilde{\mu}}_i\|^2/2\eta_i^2)}
\end{equation}
\begin{equation}
    p_{ij} = \frac{p_{j|i}+p_{i|j}}{2n}
\end{equation}
where $\widetilde{\bm{\mu}}_i$ and $\widetilde{\bm{\mu}}_j$ are the mean of the latent variables obtained from Eqn. \ref{eq:latent_mu}. Here, we select $\widetilde{\bm{\mu}}$ instead of $z$ because it contains much less noise. Term $\eta_i$ is a hyper-parameter adjusted according to the data density. One way to compute $\eta_i$ is from a pre-defined perplexity value:
\begin{equation}
    \text{Perplexity}(P_i) = 2^{-\sum_{j}p_{j|i}\log_{2}p_{j|i}}
\end{equation}

In order to encode the pairwise information, we apply a neural network to minimize the KL-divergence of the distance distribution. To be more precise, we first apply a locality mapping network that converts $\widetilde{\bm{\mu}}_i$ to $\bm{\widetilde{\bm{o}}}'_i$:
\begin{equation}
    \bm{\widetilde{\bm{o}}}'_i = f_{lp}(\bm{\widetilde{\mu}}_i; \psi)
\end{equation}
where $f_{lp}(*;\psi)$ is a neural network parameterized by $\psi$. Once $\bm{\widetilde{\bm{o}}}'_i$ is obtained, the pairwise distribution $q_{ij}$ between $\bm{\widetilde{\bm{o}}}'_i$ and $\bm{\widetilde{\bm{o}}}'_j$ is defined using a student t-distribution with one degree of freedom,
\begin{equation}
    q_{ij} = \frac{(1+||\bm{\widetilde{\bm{o}}}'_i-\bm{\widetilde{\bm{o}}}'_j||^{2})^{-1}}{\sum_k{\sum_{k\not=l}(1+||\bm{\widetilde{\bm{o}}}'_k-\bm{\widetilde{\bm{o}}}'_l||^{2})^{-1}}}
\end{equation}
t-distribution is applied due to its computational efficiency and its property to mitigate crowding problems \citep{maaten2008visualizing}. 

Finally, the local structure modeling $L_\text{LP}(x_i, x_j;\phi,\psi)$ is given by measuring the KL-divergence between the two joint probabilities $p_{ij}$ and $q_{ij}$:
\begin{equation}
    L_\text{LP}(x_i, x_j;\phi,\psi) = \sum_{i\not=j}p_{ij}\log\frac{p_{ij}}{q_{ij}}
\end{equation}
An intuitive interpretation for $L_\text{LP}(x_i;\phi,\psi)$ is that it encourages $\bm{\widetilde{\bm{o}}}'_i$ and $\bm{\widetilde{\bm{o}}}'_j$ to be closer when the pairwise distance of $\widetilde{\bm{\mu}}_i$ and $\widetilde{\bm{\mu}}_j$ are small, and vice versa.
As a result, the final output of $\bm{\widetilde{\bm{o}}}'$ not only inherits the global structure information from $\bm{\widetilde{\mu}}$ but also encodes local structure information from the pairwise relationship, thus can be used as an effective representation for clustering.


\section{Experiments}

\begin{table*}[ht]
\centering
\begin{tabular}{lllllll}
\toprule
\multirow{2}{*}{method} & \multicolumn{3}{c}{REUTERS-10k} & \multicolumn{3}{c}{REUTERS} \\
& ACC & NMI & ARI & ACC & NMI & ARI \\
\midrule
K-means      & 0.5403 & 0.4175 & 0.2775 & 0.5328 & 0.4014 & 0.2631     \\
AE+K-means   & 0.7477 & 0.4714 & 0.5300 & 0.7006 & 0.3330 & 0.3394    \\
VAE+K-means  & 0.7421 & 0.4319 & 0.5073 & 0.6922 & 0.3586 & 0.3317     \\
DEC \citep{xie2016unsupervised} & 0.7229 & 0.5229 & 0.5525 & 0.7082 & 0.4915 & 0.5387     \\
ClusterGAN \citep{mukherjee2019clustergan}  & 0.4317 & 0.0271 & 0.0163 & 0.4621 & 0.0303 & 0.0411      \\
VaDE \citep{jiang2016variational} & 0.8092 & 0.5267 & 0.5951 & 0.7938* & 0.5591 & 0.5779      \\
LPVDN (Proposed)   & \textbf{0.8301} & \textbf{0.5857} & \textbf{0.6112} & \textbf{0.8483} & \textbf{0.6553} & \textbf{0.6806}      \\
\bottomrule
\end{tabular}
\caption{Evaluation on textual datasets of REUTERS-10k and REUTERS for different algorithms using ACC, NMI and ARI. * is taken from \citep{jiang2016variational}, other results are generated by the released code.}
\label{tab:textual_result}
\end{table*}

In this section, extensive experiments are conducted to validate the effectiveness of the proposed LPVDN.

\subsection{Dataset Description}
To test the ability of our model to solve general clustering tasks, we do experiments on both vision and textual datasets.

\subsubsection{Vision Datasets:}
\begin{itemize}
    \item MNIST is a widely known dataset containing 70,000 (60,000/10,000) handwritten gray-scale images. There are 10 classes and each image size is 28 $\times$ 28 pixels.
    
    \item Fashion-MNIST: Fashion-MNIST is a dataset similar to MNIST with the same number of samples and image sizes. Instead of handwritten digits, Fashion-MNIST consists of fashion goods and products which yields a more challenging task. 
\end{itemize}

\subsubsection{Textual Datasets:}
\begin{itemize}
    \item REUTERS: The original REUTERS dataset contains about 810,000 English news stories labeled with categories. Following \citet{huang2014deep}, we use four root categories during the evaluation, namely, corporate/industrial, government/social, markets, and economics. The documents with multiple root categories are pruned, resulting in a dataset of size 685,071. 
    \item REUTERS-10k: REUTERS-10k is a subset of REUTERS.
    Using a similar approach as \citet{huang2014deep}, we adopt a small scale dataset for REUTERS by sampling a subset of 10,000 examples from REUTERS.
\end{itemize}

\begin{figure}[ht]
    \centering
    \begin{minipage}[t]{0.4\linewidth}
    \centering
    \includegraphics[width=1.4in]{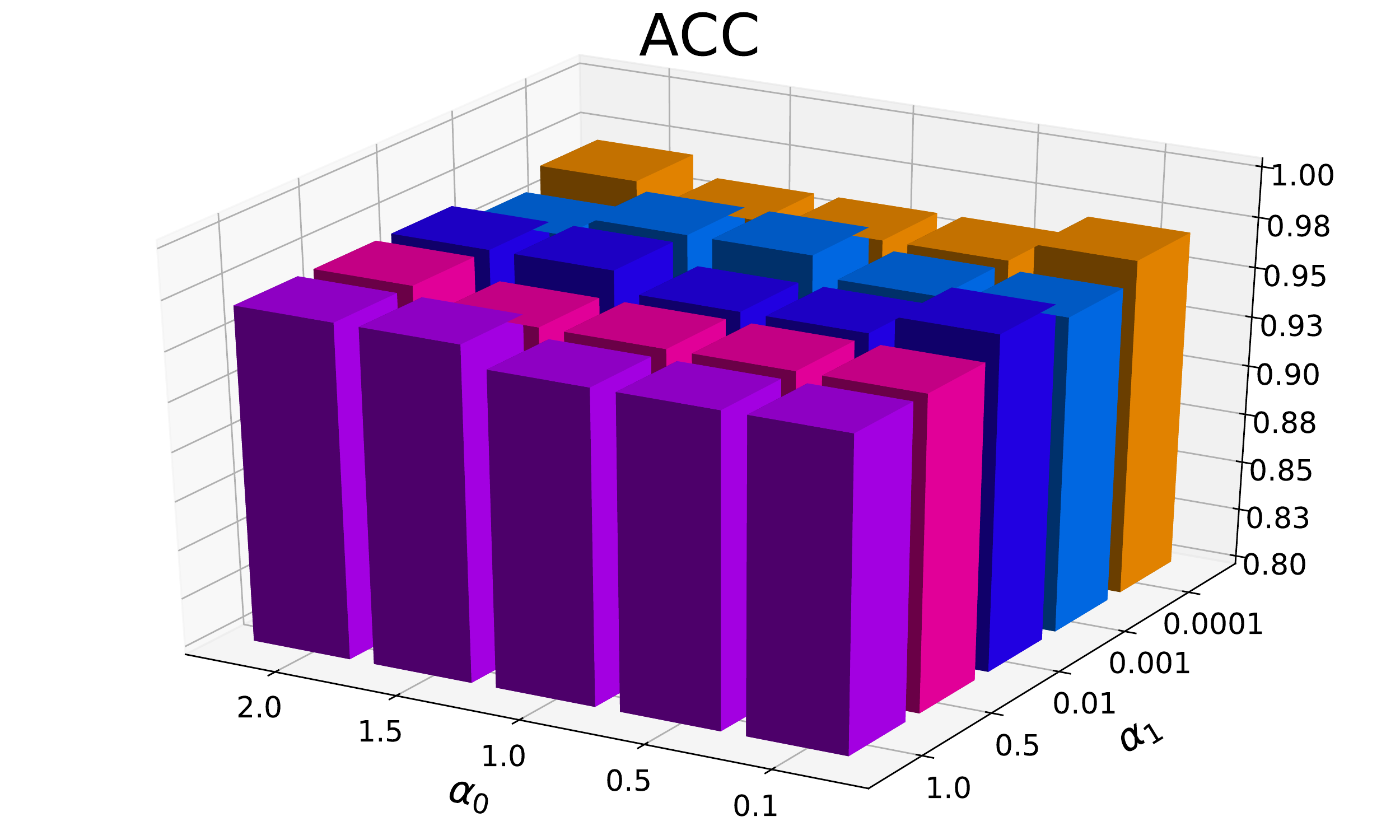}
    \label{fig:ACC}
    \end{minipage}

    \centering
    \begin{minipage}[t]{0.4\linewidth}
    \centering
    \includegraphics[width=1.4in]{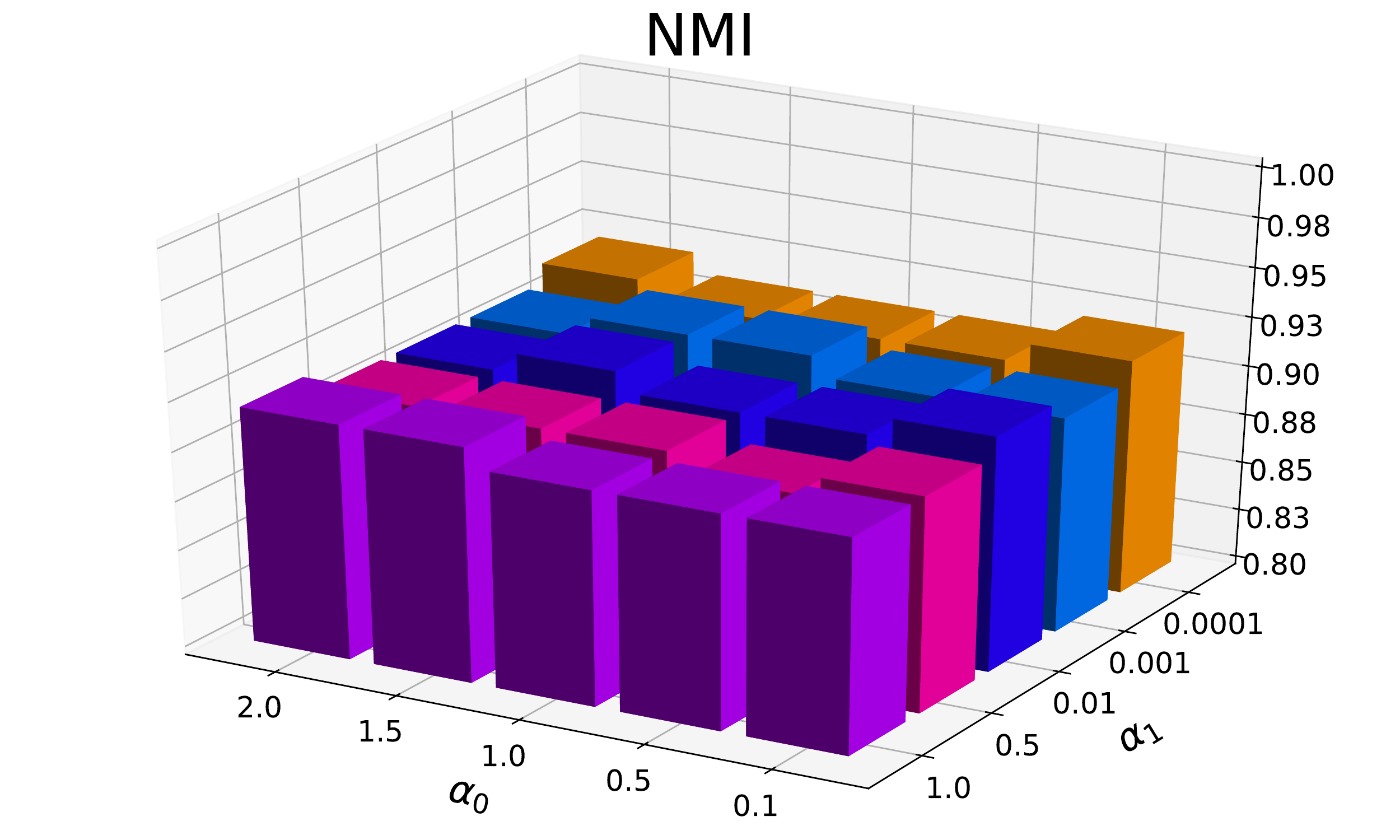}
    \label{fig:NMI}
    \end{minipage}
    \hspace{0.2in}
    \begin{minipage}[t]{0.4\linewidth}
    \centering
    \includegraphics[width=1.4in]{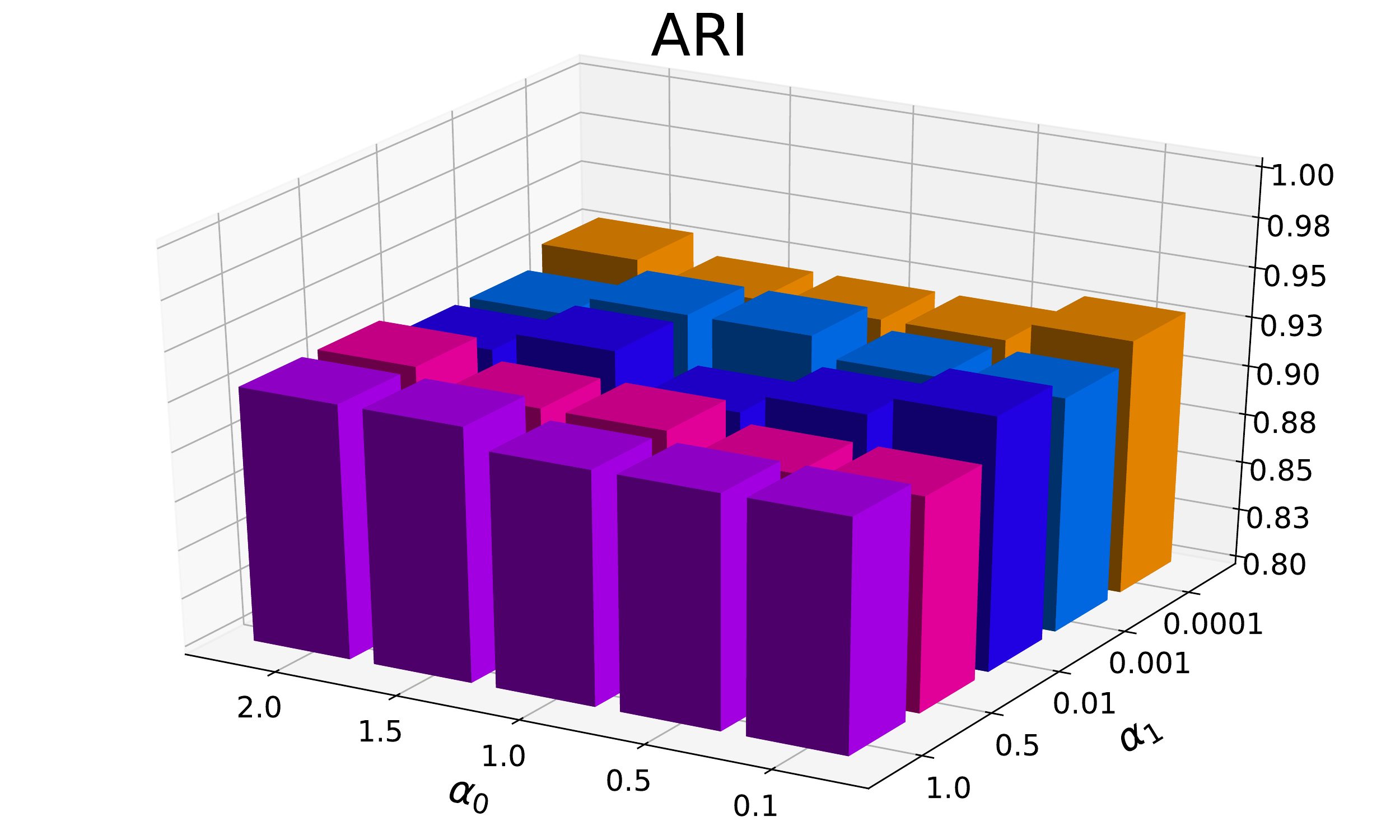}
    \label{fig:ARI}
    \end{minipage}

    \caption{ACC, NMI and ARI of the proposed LPVDN with different $\alpha_0$ and $\alpha_1$ combination on the MNIST dataset. The performance is stable with different hyper-parameters.}
    \label{fig:hyper}
\end{figure}

\subsection{Experiment Setting}

For both vision and textual datasets, we combine the training and test sets together because of the unsupervised setting. Furthermore, for vision data, we concatenate the pixels to form the input vector without preprocessing. Following the settings of \citet{xie2016unsupervised, jiang2016variational}, we set the encoder of global structure as $D-500-500-2000-10$  and decoder as $10-2000-500-500-D$, where $D$ is the input dimension. The robust discriminator uses the structure of $(10+D)-256-1$, where $10$ is the dimension of the latent representation. For the local structure model, we adopt the network with $10-256-256-256-10$. All layers are fully connected using \textit{ReLU} as the activation function. 
%
We set $\alpha_0=1$ and $\alpha_1=10^{-4}$ in MNIST and Fashion-MNIST and $\alpha_0=1$ and $\alpha_1=10^{-3}$ in REUTERS-10k and REUTERS. The batch size is 1,000. We use Adam optimizer \citep{KingmaB14}. We apply K-means on the latent representations (i.e. $\bm{\widetilde{\bm{o}}}'$) to do clustering.

\begin{table*}[t]
\centering
\begin{tabular}{lrrrrrr}
\toprule
\multirow{2}{*}{method} & \multicolumn{3}{c}{MNIST} & \multicolumn{3}{c}{REUTERS} \\
 & ACC & NMI & ARI & ACC & NMI & ARI \\
\midrule
$L_\text{G}$+$L_\text{MI}$+$L_\text{LP}$  & \textbf{0.9716} & \textbf{0.9276} & \textbf{0.9385} & \textbf{0.8483} & \textbf{0.6553} & \textbf{0.6806}  \\
$L_\text{G}$+$L_\text{LP}$ & 0.9710 & 0.9269 & 0.9373 & 0.8373 & 0.6542 & 0.6762 \\
$L_\text{G}$+$L_\text{MI}$ & 0.9477 & 0.8910 & 0.8888 & 0.8224 & 0.5969 & 0.6364 \\
$L_\text{G}$ only       & 0.9446 & 0.8761 & 0.8822 & 0.7938 & 0.5591 & 0.5779    \\

\bottomrule
\end{tabular}
\caption{Ablation study of the proposed LPVDN on MNIST and REUTERS. Both robust embedding discriminator $L_\text{MI}$ and local structure modeling $L_\text{LP}$ contribute to the improvement compared with only using global structure modeling $L_\text{G}$}
\label{tab:ablation}
\end{table*}

\subsection{Baseline Methods}
We compare the results of our model with several state-of-the-art baseline models. All the K-means algorithms mentioned below use 10 different K-means++ initialization \citep{arthur2007k} and finally select the settings with the lowest average distances. 
\begin{itemize}
    \item K-means: apply the K-means algorithm to the original data.
    \item AE + K-means: A pipeline method that trains an Autoencoder first and then applies the K-means algorithm to the latent representations produced by the Autoencoder. 
    \item VAE + K-means: A pipeline method that trains a Variational Autoencoder first and then applies K-means to the latent representations. 
    \item DEC \citep{xie2016unsupervised}: Following the experiment in the paper, DEC is first pre-trained with the autoencoder setting, and then iteratively optimized by minimizing the KL-divergence between the distribution of the current soft cluster assignment and its proposed auxiliary target distribution.
    \item ClusterGAN \citep{mukherjee2019clustergan}: introduces a variation of GAN with an inverse-mapping network and a clustering-specific loss, resulting in superior results among GAN-based algorithms in the clustering task.
    \item VaDE \citep{jiang2016variational}: a generative clustering framework combining VAE with GMM prior. It is reported that VaDE achieves the best performance compared with other general AE/VAE based clustering methods. 
    
\end{itemize}
The settings of the encoder and decoder in AE and VAE are all the same as our proposed model.

\subsection{Evaluation Metrics}

While evaluating the real performance of clustering is arguably a non-trivial problem, we follow the mainstream practices and use the unsupervised clustering accuracy (ACC), Normalized Mutual Information (NMI) and Adjusted Rand Index (ARI) as the metrics to evaluate the models. In our experiments, we set the number of clusters the same as the ground-truth labels. 

\subsubsection{ACC}
The clustering accuracy(ACC) is defined as:
\begin{equation}
    \text{ACC} = \max\limits_{m} \frac{\sum_{i=1}^{n}\textbf{1}\{l_i=m(c_i)\}}{n}
\end{equation}
where $l_i$ denotes the ground-truth label and $c_i$ is the output label of the clustering algorithm. The term $m$ ranges over all possible one-to-one mappings between the cluster and label.

\subsubsection{NMI}
Normalized Mutual Information (NMI) computes the normalized measure of similarity between two labels of the same data, which is defined as:
\begin{equation}
    \text{NMI}(U, V) = \frac{2\times \text{MI}(U, V)}{H(U) + H(V)}
\end{equation}
where $H(U)$ and $H(V)$ are the entropy of the label assignments $U$ and $V$, respectively. $\text{MI}(U, V)$ is the mutual information of the two label assignments $U$ and $V$.


\subsubsection{ARI}
The unadjusted Rand Index (RI) is a measure of the similarity between two data clusters, which is defined as
$
    (a+b)/(_{2}^{n})
$,
where $n$ is the number of samples, $a$ is the number of pairs that are divided into the same set in both ground-truth labels and output labels, $b$ is the number of pairs that are divided into different sets in both ground-truth labels and output labels.
Since RI score does not guarantee that random labels can get a score close to $0$, the Adjusted Rand Index (\textbf{ARI}) is proposed and defined as:
\begin{equation}
    \text{ARI}=\frac{\text{RI}-{E}[\text{RI}]}{\max(\text{RI})-{E}[\text{RI}]}
\end{equation}
where $E[\text{RI}]$ is the expected RI of random labeling.

\subsection{Experiment Result}
In this section, we explain the experiment results in details on vision and textual datasets.
\subsubsection{Vision Datasets:}
Table \ref{tab:visual_result} shows the results for the vision datasets, MNIST and fashion-MNIST. 

From the experiment results, we have the following observations. 1) DEC performs better than the simple combination of AE+K-means. This means that modeling the global structure during the dimension reduction is helpful to improve the clustering performance. 
2) generative methods such as VaDE and ClusterGAN achieve better results compared with discriminative methods such as DEC. This is owing to their ability to model the underlying distribution of the data instead of recovering the exact input. 
3) our method outperforms ClusterGAN and VaDE by a significant margin, this is because our model jointly learns a more robust global and local data structure, while ClusterGAN and VaDE only model the global structures.

Furthermore, we show the clustering performance by varying the hyper-parameters $\alpha_0$ and $\alpha_1$ on the MNIST dataset (in Fig. \ref{fig:hyper}). 
It can be seen that the performance of our proposed method is relatively stable in terms of different combinations of hyper-parameters $\alpha_0$ and $\alpha_1$.


\subsubsection{Textual Results:}
Table \ref{tab:textual_result} shows the results on the textual datasets. Our proposed method achieves the best results on all datasets for all the metrics. This further verifies the effectiveness of our model for non-image tasks. 
Notably, ClusterGAN  performs badly on textual datasets. As discussed in the related works, this is due to the fact that ClusterGAN is sensitive to the unbalanced data, which is the case for REUTERS. As a comparison, our method makes no assumptions on the prior distribution of the datasets, thus is not influenced by such unbalance effects.

\subsection{Ablation Study}
In Table \ref{tab:ablation}, we show the results of the ablation study on both MNIST and REUTERS datasets. As our model contains three parts ($L_\text{G}$,$L_\text{MI}$,$L_\text{LP}$),  we design the experiment to study the contribution of each component. We use the method that only models the global structure $L_G$ as the baseline to compare, which is almost equivalent to VaDE. From the experiment results, it can be seen that the local structure component ($L_\text{LP}$) is crucial and improves the clustering performance significantly, while the robust embedding discriminator $L_\text{MI}$ only improves the performance by a slight margin. This result is in-line with our expectation because the embedding discriminator is better at improving the robustness of the system rather than the accuracy.

\begin{figure}[t]
    \centering
    \subfigure{
    \begin{minipage}[t]{0.5\linewidth}
    \centering
    \includegraphics[width=1.7in]{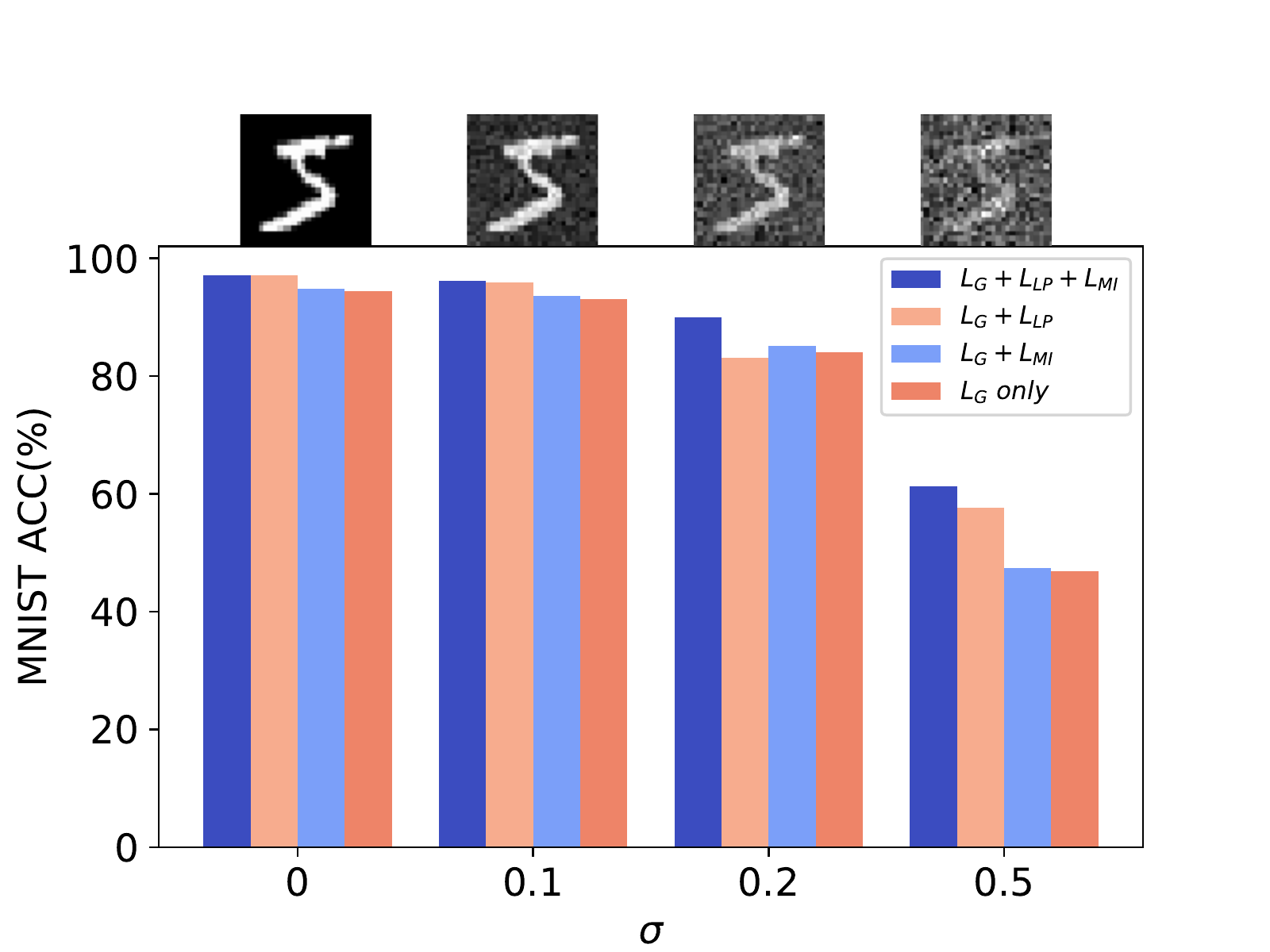}
    \end{minipage}
    }%
    \subfigure{
    \begin{minipage}[t]{0.5\linewidth}
    \centering
    \includegraphics[width=1.7in]{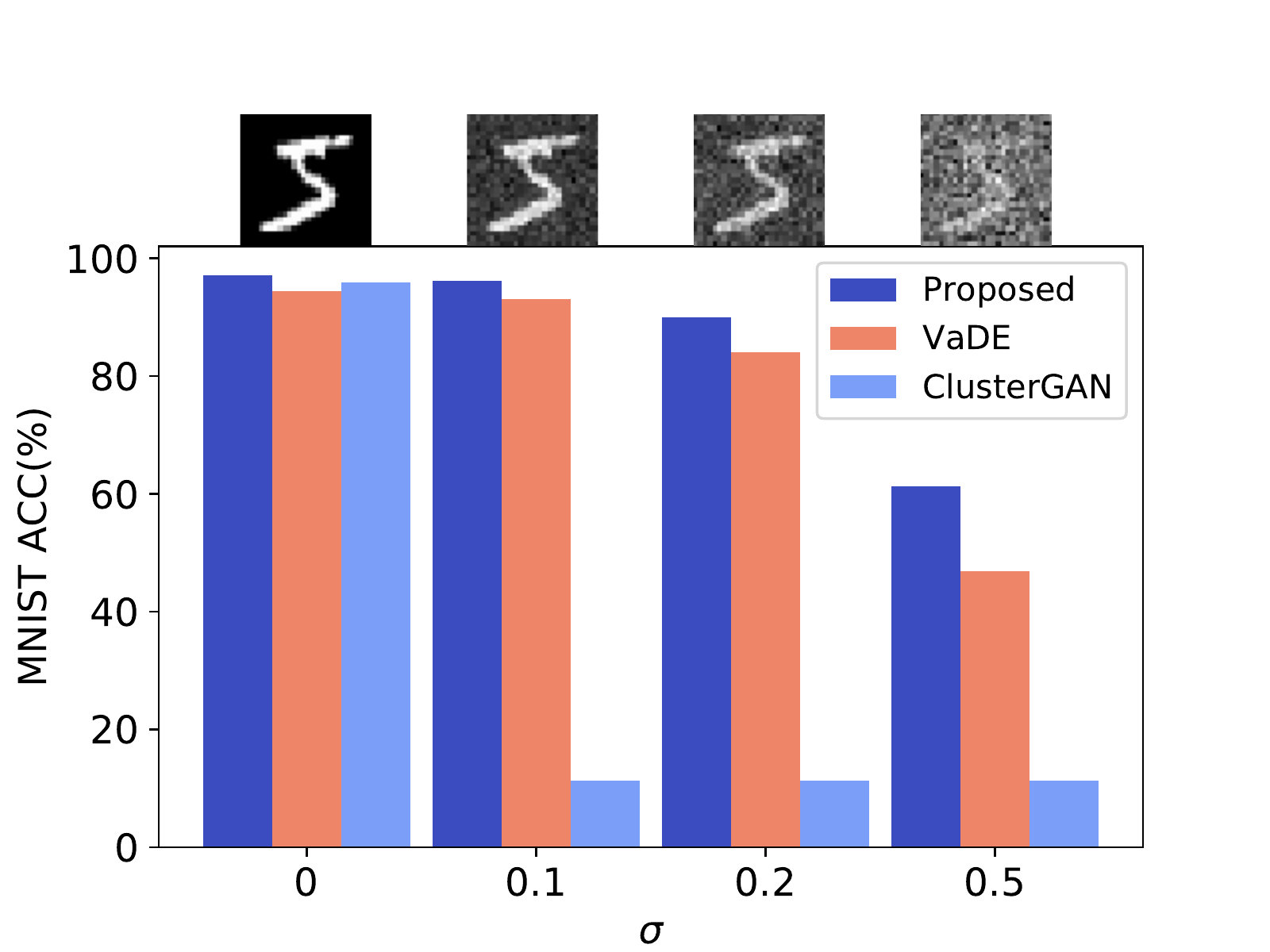}
    \end{minipage}
    }%
    \caption{Accuracy on MNIST datasets corrupted by Gaussian Noise($\sigma$). The sample images with the Gaussian noise are shown above.}\label{fig:noise}
\end{figure}

\subsection{Robustness under Input Pollution}

To testify the robustness of our model, we do experiments on the MNIST dataset following \citet{ford2019adversarial}. Particularly, we gradually corrupt the pixels in the input images with independent Gaussian Noise ($N(0,\sigma^2)$).
In Fig.\ref{fig:noise}, we show that the proposed robust embedding discriminator ($L_{MI}$) performs significantly better than the methods only using the global structure ($L_G$ only). In addition, the experiment results also show that the local structure ($L_{LP}$) contributes to the robustness, especially when the image is seriously corrupted. We can also observe that VaDE and ClusterGAN are fragile to the input noise, especially ClusterGAN, which collapses even under small pollution ratio.

\begin{figure}[t]
    \centering
    \includegraphics[width=1.8in]{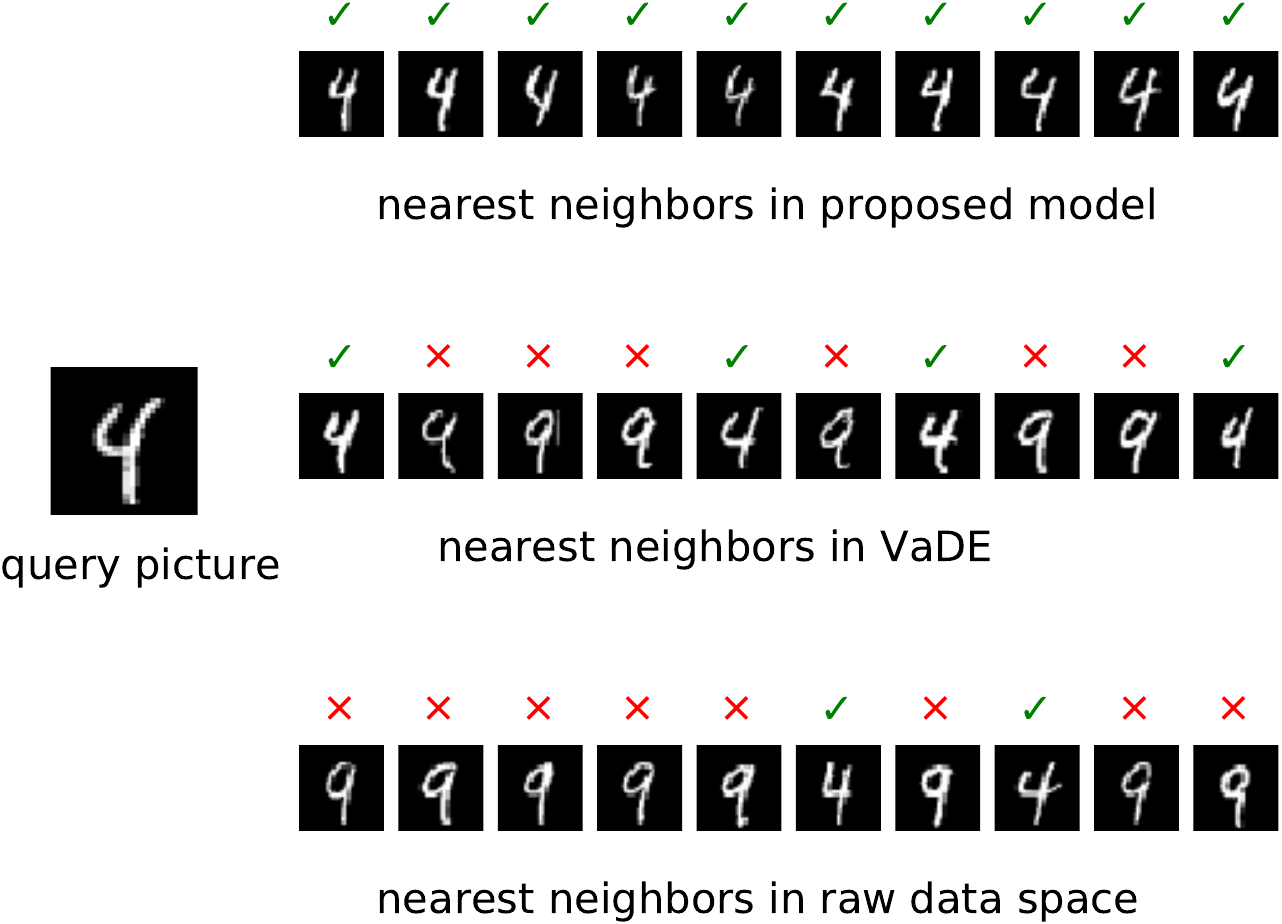}
    \caption{The 10 nearest neighbors retrieved using our proposed model, VaDE and raw data  on MNIST dataset. The picture on the left is the query picture of digit ``4'', which looks similar to ``9''.
    }\label{fig:case}
\end{figure}
\subsection{Case Study}

In Fig. \ref{fig:case}, we show a concrete example by performing an image retrieval task using the latent representations. In particular, we use the representation of an image as the query to retrieve top $K$ images based on euclidean distances.
Specifically, in the task we extract digit ``4'' form the dataset (this is a difficult example as it looks like ``9''). The top 10 nearest neighbors found by our model are all correct, while the ones obtained using VaDE are a mixture of ``4'' and ``9''. This is because the global structure modeling usually focuses on the cluster centroids and modeling the data distribution around those centroids, thus it is difficult to differentiate the data around the cluster boundary. On the other hand, since our model encodes the local pairwise relationships, it can incorporate the information of the boundary points, therefore performs better under such situations. 
\begin{figure}[t]
    \centering
    \includegraphics[width=2in]{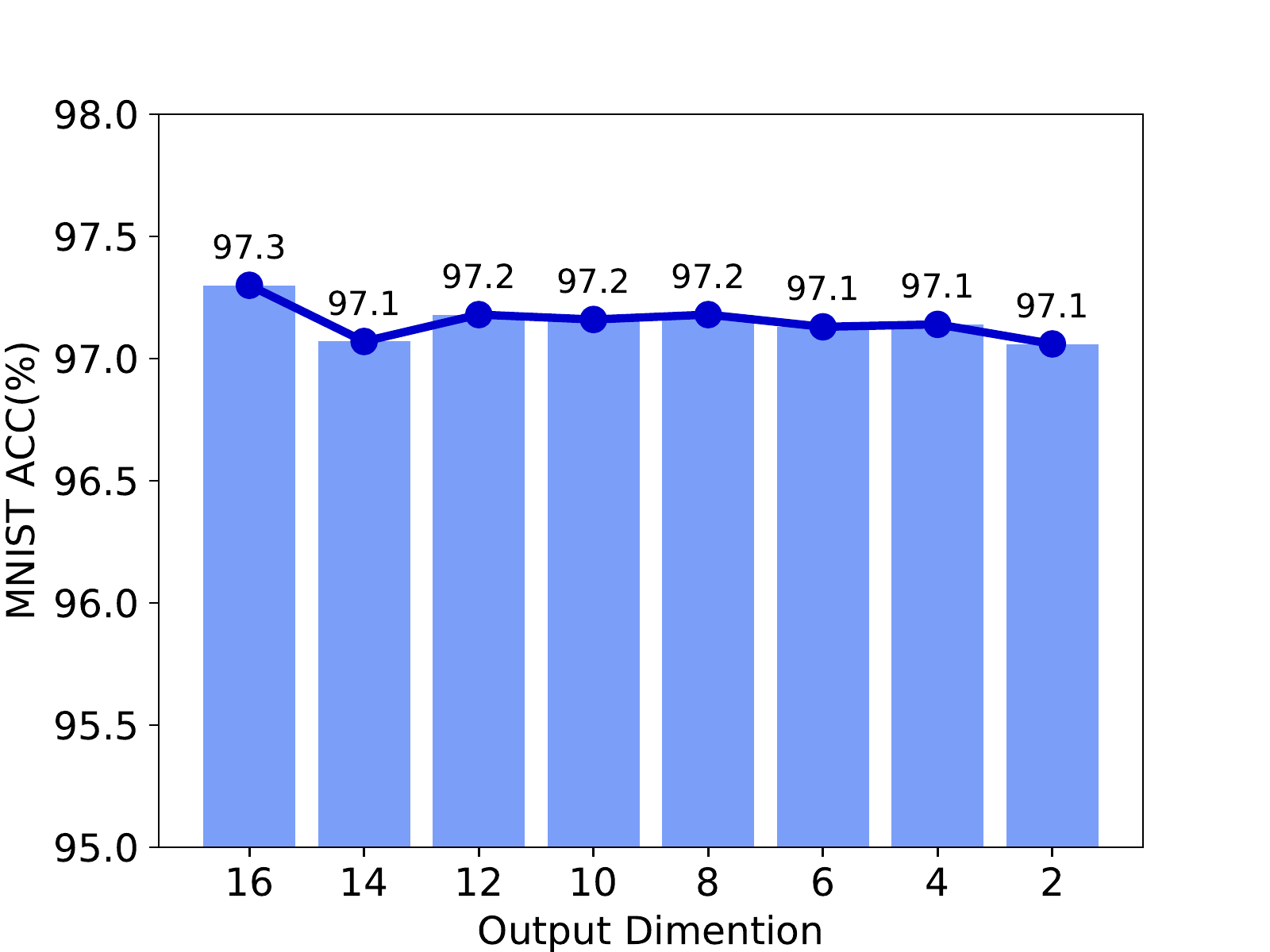}
    \caption{Accuracy on MNIST dataset with different embedding dimension sizes. Our method performs well even when the output dimension is reduced to two, which is a desirable property for data visualization.}\label{fig:case_dim}
\end{figure}

\subsection{Visualization and Dimension Reduction}

\begin{figure}[t]
    \centering
    \subfigure{
    \begin{minipage}[t]{0.4\linewidth}
    \centering
    \includegraphics[width=1.35in]{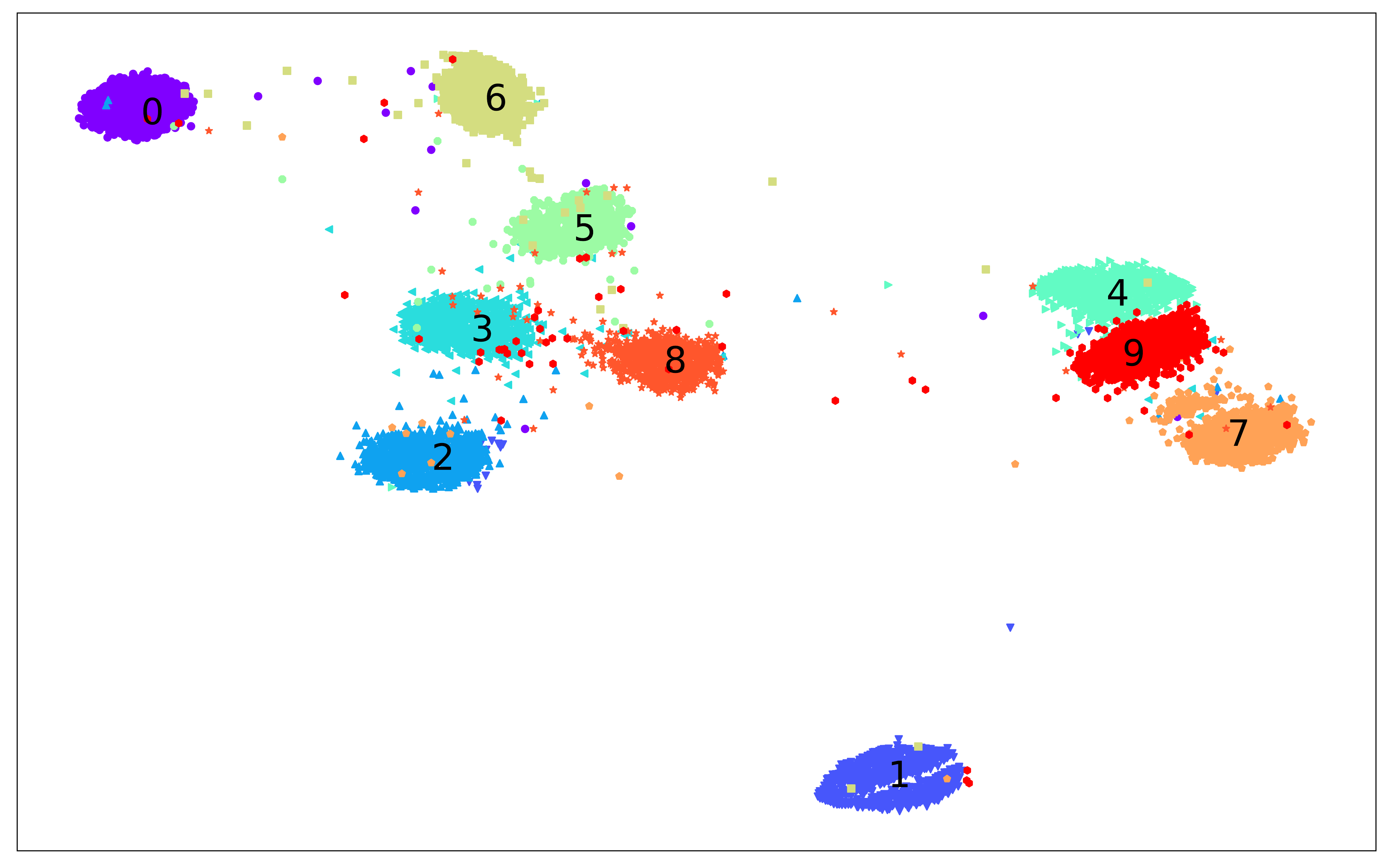}
    \end{minipage}
    }%
    \subfigure{
    \begin{minipage}[t]{0.4\linewidth}
    \centering
    \includegraphics[width=1.35in]{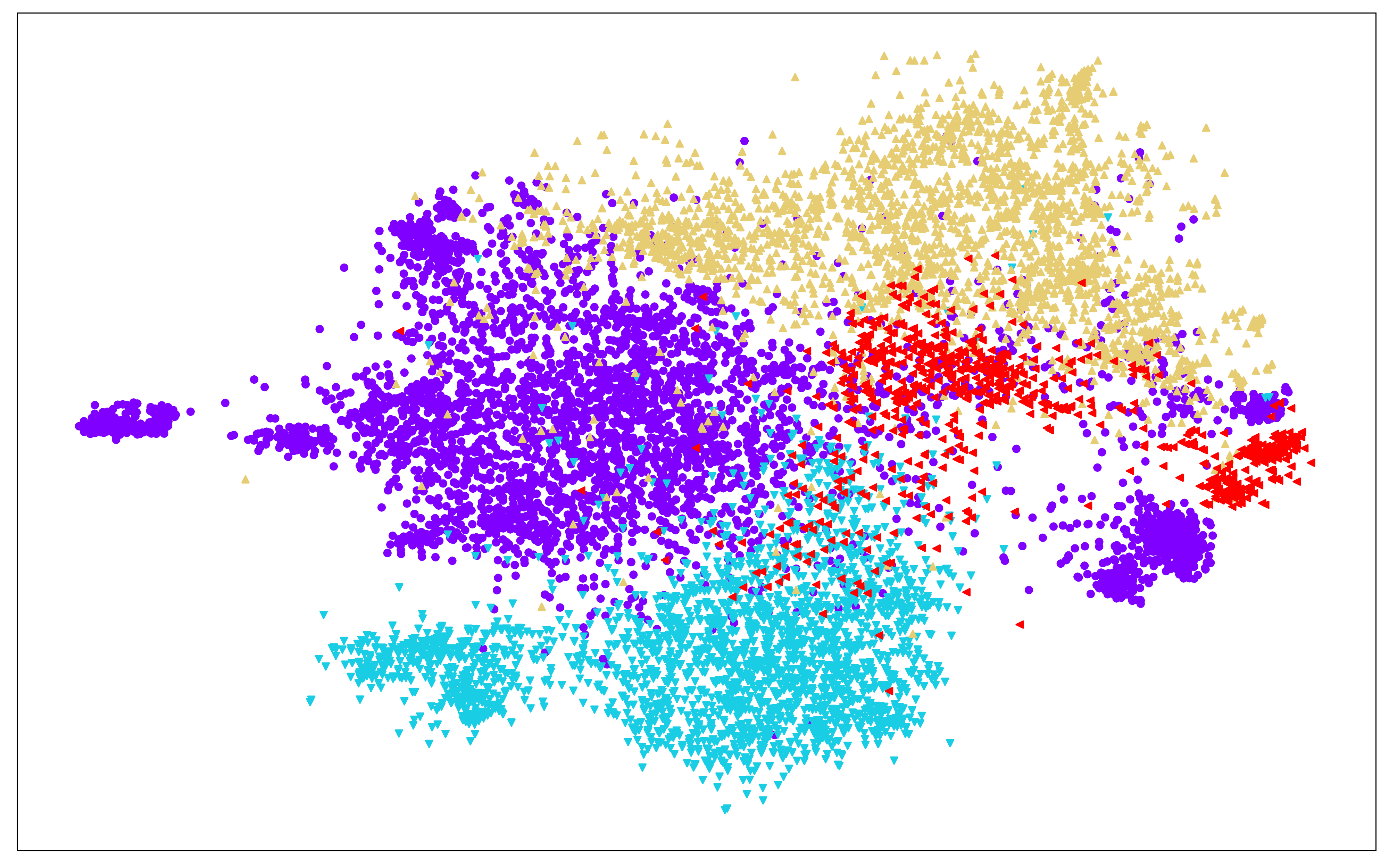}
    \end{minipage}
    }%

    \caption{Visualization on the MNIST test  dataset (left) and REUTERS-10k dataset (right) by setting the dimension of latent representations to two. The proposed LPVDN shows clear boundary among clusters. }
    \label{fig:visualization}
\end{figure}

In Fig. \ref{fig:case_dim} we examine the accuracy of our model on MNIST by adjusting the embedding dimension. It can be observed that the results are quite stable with various dimensions. One interesting observation is that the cluster accuracy is almost kept when we reduce the dimension to \textbf{2}, which makes our model promising to be served as a data visualization tool. 
In Fig. \ref{fig:visualization} we show the visualization for the latent representations on the test datasets of MNIST and REUTERS-10k, which all contain 10,000 samples. Different colors represent the gold labels. 
It shows that our proposed LPVDN creates a cohesive structure for the same clusters, while clearly separates the latent representations of different clusters.

\section{Conclusion}
In this paper, we propose a locality preserving variational discriminative network to solve the clustering problem. Because our model can capture robust global and local data structures in the latent layer, it outperforms the state-of-the-art baselines in all of the metrics for the vision and textual benchmark datasets. Moreover, our model exhibits stable clustering performance even when the dimension is reduced to 2, which means our model is a promising candidate for other unsupervised applications such as dimension reduction and data visualization.

\bibliography{aaai21}

\onecolumn
\appendix

\section{Appendix}

\subsection{Theoretical Analysis of $\mathcal{L}_\text{MI}$}

\label{sec:TheoreticalMI}

In this section, we will conduct a theoretical analysis for the global robust term $\mathcal{L}_\text{MI}$. First, we will show the relation between $\mathcal{L}_\text{MI}$ and the margin loss. Utilizing the relation between them, we replace the initial loss function with the margin loss, which can be written in the form of the penalty function. Second, We will explain the effects of $L_{MI}$ from the point of view of the penalty function. Last, we will explain how $\mathcal{L}_\text{MI}$ improve the robustness of the model. 

\subsubsection{Relation between  $\mathcal{L}_\text{MI}$ and the margin loss}

In this subsection, we will analysis the relation between $\mathcal{L}_\text{MI}$ and the margin loss.

The form of $\mathcal{L}_\text{MI}$ can be written as:
\begin{align}
    \mathcal{L}_{\text{MI}}&=-\mathbb{E}_{x, z\sim q(z|x)p(x)}[\log(\sigma(D(x, z))]-\mathbb{E}_{x, z\sim q(z)p(x)}[\log(1-\sigma(D(x, z))]\\
    &= \mathbb{E}_{x, z\sim q(z|x)p(x)}[\log(1+\exp(-D(x, z)))]+\mathbb{E}_{x, z\sim q(z)p(x)}[\log(1+\exp(D(x, z)))]
\end{align}

The function $f(t)=\log(1+\exp(t))$ can be approximated by a margin loss $g(t)=[t+\gamma]_+=\max(0, t+\gamma) (\gamma \ge 0)$, which is derivable almost everywhere. When the model is well-learned, namely the loss function $f(t)$ is small enough, their Chebyshev distance $D$ and their gradients' Chebyshev distance are 
\begin{align}
    D=\inf_{t<-\gamma}|f(t)-g(t)|=\log(1 + \exp(-\gamma)); \quad D'=\inf_{t<-\gamma}|f'(t)-g'(t)|=\frac{\exp(-\gamma)}{1+\exp(-\gamma)}
\end{align}

Therefore, we can choose a appropriate $\gamma$ to ensure that if we replace $f(t)$ with $g(t)$, the gradients and well-learned loss values are precise enough. For example, when $\gamma=3$, $D=\log(1+e^{-3})=0.049<0.05, D'=e^{-3}/(1+e^{-3})=0.047<0.05$, the well-learned loss values and gradients have bounded errors less than $0.05$.

\subsubsection{Equivalence between  $\mathcal{L}_\text{MI}$ and the penalty function}

In this subsection, we will analysis the equivalence between $\mathcal{L}_\text{MI}$ and the penalty function under certain circumstances, which improve the robustness of the model.

After replacing $f(t)$ with $g(t)$, 
\begin{align}
\mathcal{L}_{\text{MI}} = \mathbb{E}_{ q(z|x)p(x)}[f(-D(x, z))]+\mathbb{E}_{ q(z)p(x)}[f(D(x, z))]
\approx \mathbb{E}_{ q(z|x)p(x)}[[\gamma-D(x, z)]_+]+\mathbb{E}_{ q(z)p(x)}[[\gamma+D(x, z)]_+]
\end{align}

Let $x_i, z_i\sim q(z|x)p(x) (1\le i\le n)$ be the positive samples and $x_j, z_j\sim q(z)p(x) (n+1\le j \le n+m)$ be the negative samples for training $\mathcal{L}_\text{MI}$, when  minimizing the total loss function $\mathcal{L}+\lambda \mathcal{L}_\text{MI}$, where $\mathcal{L}$ is the sum of other loss functions, the optimization form can be written as:
\begin{align}
    \min\quad  \mathcal{L}+\alpha{\sum\limits_{i=1}^{n}[\gamma-D(x_i, z_i)]_++\beta{\sum\limits_{j=n+1}^{n+m}}[\gamma+D(x_j, z_j)]_+};
\end{align}
which is equivalent to the following optimization under certain circumstances with the point of the penalty function,
\begin{align}
\min&\quad \mathcal{L}\\
s.t.&\quad [\gamma-D(x_i, z_i)]_+=0; (1\le i\le n)  \\
&\quad [\gamma+D(x_j, z_j)]_+=0; (n+1\le j \le n+m)
\end{align}

Therefore, for $x, z\sim q(z|x)p(x)$, $D(x, z)>\gamma$ holds with high probability, and for $x,z\sim q(z)p(x)$, $D(x, z)<-\gamma$ holds with high probability.

\subsubsection{How $\mathcal{L}_\text{MI}$ improves the robustness of the model}

In this subsection, we will show how $\mathcal{L}_\text{MI}$ improves the robustness of the model.

First, we define the rationality of hidden state $z$. Define $X$ as the manifold of real $x$, with a well trained discriminator $D$, $D(x, z)>0$ is the rational region for distribution $q(z|x)p(x)$. With the point of the penalty function, we have 
\begin{align}
\alpha(x)=\int\limits_{D(x,z)>0}q(z|x)dz \approx \int\limits_{D(x,z)>\gamma}q(z|x)dz  
\end{align}
So for $x$ in manifold $X$, $\alpha(x) \approx 1$ if $D$ is well trained. And $\alpha(x) \approx 1$ infers that the hidden state generated by $x$ is rational.

Define $x'=x+\Delta x, x\in X$ being the data out of manifold $X$, when we use a continuous function $q(z|x)$, we will prove that the generated hidden state $z'$ is still rationality with high probability when $\Delta x$ is small. 
\begin{align}
    \alpha (x+\Delta x)= \int\limits_{D(x+\Delta x,z)>0}q(z|x+\Delta x)dz = \int\limits_{D(x,z)+\nabla_xD\cdot\Delta x>0}q(z|x)dz + \int\limits_{D(x,z)>0}\nabla_x[q(z|x)]\cdot\Delta x dz + o(\|\Delta x\|_2)
\end{align}

where high-order infinitesimals are ignored, and we will show in Corollary~\ref{cor:1} that $\int\limits_{D(x,z)>0}\nabla_x[q(z|x)]\cdot\Delta x dz$ can also be ignored for some certain distributions.

Therefore, we only need to consider the first term. Define $m=-\nabla_xD\cdot\Delta x$, when $m>0$, $\alpha (x')$ tends to get worse.
\begin{align}
\alpha(x+\Delta x)= \int\limits_{D(x,z)>0}q(z|x)dz - \int\limits_{D(x,z)\in(0, m)}q(z|x)dz+o(\|\Delta x\|_2) =\alpha(x)-\int\limits_{D(x,z)\in(0, m)}q(z|x)dz+o(\|\Delta x\|_2)
\end{align}
When $m<\gamma$, $\mathcal{L}_\text{MI}$ makes $D(x,z)$ almost impossible to be less than $\gamma$, therefore $\int\limits_{D(x,z)\in(0, m)}q(z|x)dz$ is almost $0$. And $\alpha(x+\Delta x)$ tends to decrease without $\mathcal{L}_\text{MI}$ because $\int\limits_{D(x,z)\in(0, m)}q(z|x)dz$ tends to be greater than $0$. 

To conclude, $\mathcal{L}_\text{MI}$ improves the robustness of the model.

\subsubsection{Corollary~\ref{cor:1} and its proof}
\begin{cor}
\label{cor:1}
For $q(z|x)=\frac{1}{(2\pi)^{d_z/2}\prod_{i=0}^{d_z-1}\sigma_i}\exp{[-\frac{1}{2}\sum_{i=0}^{d_z-1}\frac{(z_i-\mu_i(x))^2}{\sigma_i^2}]}$, where $\mu_i(x) (i=0,1,\cdots,d_z-1)$ is a smooth function and $d_z$ is the dimension of $z$, the term $\int\limits_{D(x,z)>0}\nabla_x[q(z|x)]\cdot\Delta x dz$ is a high-order infinitesimal and can be ignored, namely
\begin{align}
\int\limits_{D(x,z)>0}\nabla_x[q(z|x)]\cdot\Delta x dz=o(\|\Delta x\|_2)
\end{align}
\end{cor}

\begin{proof}


Since $D(x, z)$ measures whether $z$ is generated from $q(z|x)$ with high probability, we may assume that the closed region $\mathcal{A}=\{z:\sum_{i=0}^{d_z-1}\frac{(z_i-\mu_i(x))^2}{\sigma_i^2}<\epsilon^2\}$ can approximate the region $\{z: D(x,z)>0\}$ with only an $o(1)$ infinitesimal ignored. Therefore,

\begin{align}
     &\quad \int\limits_{D(x,z)>0}\nabla_x[q(z|x)]\cdot\Delta x dz = \int\limits_{D(x,z)>0}q(z|x)\sum_{i=0}^{d_z-1}\big(\frac{z_i-\mu_i(x)}{\sigma_i^2}\mu_i'(x)\Delta x_i \big)dz\\ 
     &=\sum_{i=0}^{d_z-1}\big( \int\limits_{D(x, z)>0}q(z|x)(z_i-\mu_i(x))dz \times \frac{\mu_i'(x)}{\sigma_i^2}\Delta x_i\big)
     =\sum_{i=0}^{d_z-1}\big([ \int\limits_{\mathcal{A}}q(z|x)(z_i-\mu_i(x))dz+o(1)] \times \frac{\mu_i'(x)}{\sigma_i^2}\Delta x_i\big)\\
    &=\sum_{i=0}^{d_z-1}\big(o(1) \times \frac{\mu_i'(x)}{\sigma_i^2}\Delta x_i\big)=o(\|\Delta x\|_2)
\end{align}

\end{proof}

\subsection{Experimental Settings}
In this section, we illustrate the details of the experiment in the paper. The experimental settings are shown in Table~\ref{tab:1}.

\begin{table*}[!h]
\centering
\begin{tabular}{lrrrr}
\toprule
 & MNIST & Fashion-MNIST & REUTERS10k & Reuters \\
\midrule
$\alpha_0$ & 1 & 1 & $10^{-2}$ & 1 \\
$\alpha_1$ & $10^{-4}$ & $10^{-4}$ & $10^{-3}$ & $10^{-2}$ \\
batch size & 800 & 800 & 1000 & 1000 \\
epoch number & 300 & 300 & 50 & 300 \\
optimizer & adam & adam & adam & adam \\
learning rate & $2\times 10^{-3}$ & $2\times 10^{-3}$ & $2\times 10^{-4}$ & $2\times 10^{-4}$ \\
lr-decay & 0.95/10epoch & 0.95/10epoch & 0.95/10epoch & 0.95/10epoch \\
\bottomrule
\end{tabular}
\caption{Experimental settings.}
\label{tab:1}
\end{table*}

\end{document}